\documentclass{daj}
\usepackage{amsthm}

\usepackage{hyperref}
\usepackage[slantedGreeks, partialup, noDcommand]{kpfonts}
\usepackage{graphicx}
\usepackage{caption}
\usepackage{tikz}
\usetikzlibrary{trees}

\newtheorem{theorem}{Theorem}
\newtheorem{lemma}[theorem]{Lemma}

\newtheorem{proposition}[theorem]{Proposition}
\newtheorem{corollary}[theorem]{Corollary}

\newtheorem{remark}[theorem]{Remark}

\newcommand{\eqdef}{\stackrel{\mathrm{def}}{=}}

\newcommand{\R}{\mathbb{R}} 

\newcommand{\N}{\mathbb{N}}

\newcommand{\e}{\varepsilon}

\usepackage{dutchcal}

\newcommand{\XX}{\mathcal{X}}

\newcommand{\E}{\mathbb{E}}

\newcommand{\PP}{\mathbb{P}}

\newcommand{\mb}{\mathbb}
\newcommand{\ms}{\mathscr}
\newcommand{\msf}{\mathsf}

\newcommand*\circled[1]{\tikz[baseline=(char.base)]{
            \node[shape=circle,draw,inner sep=0pt] (char) {#1};}}

\dajAUTHORdetails{%
  title = {Low-Degree Learning and the Metric Entropy of Polynomials}, 
  author = {Alexandros Eskenazis, Paata Ivanisvili, and Lauritz Streck},
  plaintextauthor = {Alexandros Eskenazis, Paata Ivanisvili, and Lauritz Streck},
  keywords = {Discrete hypercube, learning theory, query complexity, metric entropy, junta, constant depth circuit.},
}   

\dajEDITORdetails{%
   year={2023},
   number={17},
   received={30 August 2022},   
   published={27 November 2023},  
   doi={10.19086/da.88507},       
}   

\begin{document}

\begin{frontmatter}[classification=text]

\title{Low-Degree Learning and the Metric Entropy of Polynomials}

\author[alex]{Alexandros Eskenazis\thanks{Supported by a Junior Research Fellowship from Trinity College, Cambridge.}}
\author[paata]{Paata Ivanisvili\thanks{Supported by the NSF grants DMS-2152346 and CAREER-DMS-2152401.}}
\author[lauritz]{Lauritz Streck\thanks{Supported by the European Research Council (ERC)
under the European Union’s Horizon 2020 research and innovation programme (grant agreement No. 803711). .}}

\begin{abstract}
Let $\ms{F}_{n,d}$ be the class of all functions $f:\{-1,1\}^n\to[-1,1]$ on the $n$-dimensional discrete hypercube of degree at most $d$. In the first part of this paper, we prove that any (deterministic or randomized) algorithm which learns $\ms{F}_{n,d}$ with \mbox{$L_2$-accuracy} $\e$  requires at least $\Omega((1-\sqrt{\e})2^d\log n)$ queries for large enough $n$, thus establishing the sharpness as $n\to\infty$ of a recent upper bound of Eskenazis and Ivanisvili (2021). To do this, we show that the $L_2$-packing numbers $\msf{M}(\ms{F}_{n,d},\|\cdot\|_{L_2},\e)$ of the concept class $\ms{F}_{n,d}$ satisfy the two-sided estimate
$$c(1-\e)2^d\log n \leq \log \msf{M}(\ms{F}_{n,d},\|\cdot\|_{L_2},\e) \leq \frac{2^{Cd}\log n}{\e^4}$$
for large enough $n$, where $c, C>0$ are universal constants.  In the second part of the paper, we present a logarithmic upper bound for the randomized query complexity of classes of bounded \emph{approximate} polynomials whose Fourier spectra are concentrated on few subsets.  As an application, we prove new estimates for the number of random queries required to learn approximate juntas of a given degree, functions with rapidly decaying Fourier tails and constant depth circuits of given size. Finally, we  obtain bounds for the number of queries required to learn the polynomial class $\ms{F}_{n,d}$ without error in the query and random example models. \let\thefootnote\relax\footnotetext{\emph{\\ Updated version of Discrete Analysis 2023:17 with typographical error in Corollary \ref{cor:circuits} corrected.}}
\end{abstract}
\end{frontmatter}



\section{Introduction}

For any function $f:\{-1,1\}^n\to\R$, there exist unique real coefficients $\{\hat{f}(S)\}_{S\subseteq\{1,\ldots,n\}}$ such that
\begin{equation}
\forall \ x\in\{-1,1\}^n, \qquad f(x) = \sum_{S\subseteq\{1,\ldots,n\}} \hat{f}(S) w_S(x),
\end{equation}
where the Walsh functions $w_S:\{-1,1\}^n\to\{-1,1\}$ are defined by $w_S(x)=\prod_{i\in S}x_i$. We say that the function $f$ has degree at most $d\in\{0,1,\ldots,n\}$ if $\hat{f}(S)=0$ for all sets $S$ with $|S|>d$.

\vspace{-1.4mm}
\subsection{Learning functions on the hypercube} Let $\ms{F}$ be a class of functions on the discrete hypercube of dimension $n$. The learning problem for the class $\ms{F}$ can be described as follows. Consider an unknown function $f\in\ms{F}$. Given access to \emph{examples} $(X_1,f(X_1)), \ldots, (X_Q, f(X_Q))$, the goal is to algorithmically construct a \emph{hypothesis function} $h:\{-1,1\}^n\to\R$ which effectively approximates $f$. Different access models to examples give rise to concrete versions of the learning problem. The two most standard such models are the \emph{query} model, in which the algorithm can sequentially request the values of $f$ at any $Q$-tuple of points $X_1,\ldots,X_Q$ from $\{-1,1\}^n$, and the \emph{random example} model, in which the data points $X_1,\ldots,X_Q$ are generated uniformly and independently from $\{-1,1\}^n$. In the query model, the goal is to construct a function $h$ satisfying $\|h-f\|^2_{L_2}\leq\e$ whereas in the random example model, the desired output is a \emph{random} function $h$ satisfying $\|h-f\|^2_{L_2}\leq\e$ with probability at least $1-\delta$, where $\e, \delta \in[0,1)$ are pre-fixed accuracy and confidence parameters respectively. The least number $Q$ of examples required to solve the learning problem in each case is called the \emph{query complexity} of the model and shall be denoted by $\msf{Q}(\ms{F},\e)$ for the query model and\mbox{ by $\msf{Q}_r(\ms{F},\e,\delta)$ for the random example model.}

The query complexity of learning problems has been studied extensively for various classes $\ms{F}$ of functions on the discrete hypercube (see \cite{Man94, O'Do14}). One of the first rigorous results of this kind is the \emph{Low-Degree Algorithm} of Linial, Mansour and Nisan \cite{LMN93}, who considered the class
\begin{equation} \label{eq:Fnd}
\ms{F}_{n,d} \eqdef \big\{f:\{-1,1\}^n\to[-1,1]: \ f \mbox{ has degree at most } d\big\}
\end{equation}
and showed the estimate $\msf{Q}_r(\ms{F}_{n,d},\e,\delta) \leq \frac{2n^d}{\e}\log(\tfrac{2n^d}{\delta})$ for any $\e,\delta\in(0,1)$. In the recent work \cite{EI21}, which followed an intermediate $O_{d,\e,\delta}(n^{d-1}\log n)$ asymptotic\footnote{We shall use the standard asymptotic notation throughout the article. For $a,b>0$, we write $a=O(b)$ or $b=\Omega(a)$ if there exists a universal constant $c>0$ such that $a\leq cb$. Moreover, we shall write $a=\Theta(b)$ for $a=O(b)$ and $b=O(a)$. We shall also write $O_\xi(\cdot)$, $\Theta_\psi(\cdot)$ to indicate that the implicit constants depend on $\xi$ or $\psi$ respectively.} improvement in \cite{IRRRY21}, it was shown that this classical estimate is largely suboptimal as $n\to\infty$ and in fact
\begin{equation} \label{eq:EI}
\msf{Q}_r(\ms{F}_{n,d},\e,\delta) \leq \min\left\{ \frac{\exp(C d^{3/2}\sqrt{\log d})}{\e^{d+1}},\frac{4dn^d}{\e}\right\} \log\left(\frac{n}{\delta}\right)
\end{equation}

The first goal of the present paper is to investigate lower bounds for the query complexity, which in particular imply that \eqref{eq:EI} is asymptotically optimal as $n\to\infty$, that is
\begin{equation} \label{eq:equiv}
\msf{Q}_r(\ms{F}_{n,d},\e,\delta) = \Theta_{d,\e,\delta}(\log n).
\end{equation}
Consider the class
\begin{equation} \label{eq:Bnd}
\ms{B}_{n,d} \eqdef \big\{f:\{-1,1\}^n\to\{-1,1\}: \ f \mbox{ has degree at most } d\big\}\subset \ms{F}_{n,d}.
\end{equation}
We will prove the following lower estimate for the complexities of this class.

\begin{theorem} \label{thm:lb}
For any $n\in\N$, $d\in\{1,\ldots,n\}$,  $\e\in[0,1)$ and $\delta\in(0,1)$, we have
\begin{equation} \label{eq:EIS-det}
\msf{Q}(\ms{B}_{n,d},\e) \geq \max\left\{(1-\sqrt{\e})2^{d-2}\log_2 n - (d+1)2^{d-2},  d\log_2\left(\frac{n}{d}\right)\right\}
\end{equation}
and
\begin{equation} \label{eq:EIS}
\msf{Q}_r(\ms{B}_{n,d},\e,\delta) \geq \max\left\{(1-\sqrt{\e})2^{d-2}\log_2 n - (d+1)2^{d-2},  d\log_2\left(\frac{n}{d}\right)\right\}+\log_2(1-\delta).
\end{equation}
\end{theorem}
\noindent The equivalence \eqref{eq:equiv} now follows due to the inequality $\msf{Q}_r(\ms{B}_{n,d},\e,\delta)\leq\msf{Q}_r(\ms{F}_{n,d},\e,\delta)$. The tools used in the proof of Theorem \ref{thm:lb} will be described in Section \ref{sec:1.3}.

The second goal of this paper is to show a more robust version of the upper bound \eqref{eq:EI} that applies to different concept classes $\ms{F}$ which are not necessarily of bounded degree.  In order to present this result we shall need some terminology (see also \cite[Chapter~3]{O'Do14}).  If $t\geq0$, we say that the Fourier spectrum of a function $f:\{-1,1\}^n\to\R$ is $t$-concentrated up to degree $d$ if
\begin{equation} \label{tail}
\sum_{|S|>d}\hat{f}(S)^2\leq t.
\end{equation}
More generally, given a family $\ms{S}$ of subsets of $\{1,\ldots,n\}$ we say that the spectrum of $f$ is $\eta$-concentrated on $\ms{S}$ if
\begin{equation}
\sum_{S\notin\ms{S}} \hat{f}(S)^2 \leq \eta.
\end{equation}
Our main upper bound for learning is the following theorem.

\begin{theorem} \label{thm:upper}
Fix $n,m\in\N$, $d\in\{1,\ldots,n\}$ and $t,\eta\in[0,1)$.  Let $\ms{F}$ be a class of bounded functions $f:\{-1,1\}^n\to[-1,1]$ such that the Fourier spectrum of any $f\in\ms{F}$ is $t$-concentrated up to degree $d$ and is $\eta$-concentrated on a family $\ms{S}(f)$ of subsets of $\{1,\ldots,n\}$ satisfying $\#\ms{S}(f)\leq m$. Then,
\begin{equation} \label{eq:upper}
\forall \ \e,\delta\in(0,1),\qquad \msf{Q}_r(\ms{F},\eta+t+\e,\delta) \leq \left\lceil\frac{18m}{\e}\log\left(\frac{2}{\delta} \sum_{r=0}^d\binom{n}{r}\right)\right\rceil.
\end{equation}
\end{theorem}

In Remark \ref{rem:newproof} below we will see how this statement implies the estimate \eqref{eq:EI} of \cite{EI21}.


\subsection{Fourier concentration and learning upper bounds} \label{sec:upper} In this section we shall present concrete applications of Theorem \ref{thm:upper} for various concept classes $\ms{F}$.  We start with the class of approximate juntas of a given degree. Recall that a function \mbox{$f:\{-1,1\}^n\to\R$} is called a $(k,\eta)$-junta if there exists a subset $\sigma\subseteq\{1,\ldots,n\}$ with $|\sigma|\leq k$ and a map \mbox{$g:\{-1,1\}^n\to\R$} depending only on the variables $(x_i)_{i\in\sigma}$ such that $\|f-g\|^2_{L_2}\leq \eta$. Consider the class
\begin{equation}
\ms{J}_{n,k,\eta} = \big\{f:\{-1,1\}^n\to[-1,1]: \ f \mbox{ is a } (k,\eta)\mbox{-junta} \big\}.
\end{equation}
We shall prove the following estimate on the randomized query complexity of $\ms{F}_{n,d}\cap\ms{J}_{n,k,\eta}$.

\begin{corollary} \label{thm:junta}
In the setting above, for $\e,\delta\in(0,1)$ we have
\begin{equation} \label{eq:junta}
\msf{Q}_r\big(\ms{F}_{n,d} \cap \ms{J}_{n,k,\eta}, 2\eta+\e,\delta\big) \leq \left\lceil \frac{18}{\e} \sum_{r=0}^{\min\{d,k\}} \binom{k}{r} \log\left(\frac{2}{\delta}\sum_{r=0}^{\min\{d,k\}} \binom{n}{r}\right) \right\rceil.
\end{equation}
In particular, choosing $d=n$, we get
\begin{equation}
\msf{Q}_r\big(\ms{J}_{n,k,\eta}, 2\eta+\e,\delta\big) \leq  \frac{2^{k+5}}{\e} \log\left(\frac{2}{\delta}\sum_{r=0}^{k} \binom{n}{r}\right).
\end{equation}
\end{corollary}

Corollary \ref{thm:junta} can be concretely applied in view of the large available literature on junta theorems in Boolean analysis.  To motivate a first application along these lines, observe that the upper bound \eqref{eq:EI} of \cite{EI21} differs from that of \cite{LMN93} in its dependence on $\e$ as $\e\to0^+$.  While we do not know whether the $\e^{-d-1}$ asymptotic behavior is needed to learn $\ms{F}_{n,d}$, Corollary \ref{thm:junta} combined with a  structural result of Nisan and Szegedy \cite{NS94} gives the following upper bound for the complexity of $\ms{B}_{n,d}$, alas with a somewhat worse dependence on $d$.

\begin{corollary} \label{prop:boolean}
For any $n\in\N$, $d\in\{1,\ldots,n\}$ and $\e,\delta\in(0,1)$, we have
\begin{equation}
\msf{Q}_r(\ms{B}_{n,d},\e,\delta) \leq  \frac{36\cdot d2^{d^2}}{\e} \log\left(\frac{n}{\delta}\right).
\end{equation}
\end{corollary}

Combining Theorem \ref{thm:upper} with a deep junta theorem of Dinur, Friedgut, Kindler and O'Donnell \cite{DFKO07}, we will deduce that bounded functions which are sufficiently close to polynomials of degree $d$ can be learned from $O_d(\log n)$ samples. For $t\geq0$, consider the class $\ms{F}_{n,d}(t)$ consisting of functions $f:\{-1,1\}^n\to[-1,1]$ whose spectra are $t$-concentrated up to degree $d$.
In other words, $\ms{F}_{n,d}(t)$ consists of all bounded functions which are $\sqrt{t}$-close (in $L_2$) to a polynomial of degree at most $d$. Corollary \ref{thm:junta} has the following consequence.

\begin{corollary} \label{cor:robust}
There exists a universal constant $C>0$ such that the following holds. For any $n\in\N$, $d\in\{1,\ldots,n\}$, $t\in[0,1)$ and $\eta\geq\frac{C d^2{\log d}}{{\log(1/t)}}$, we have
\begin{equation} \label{eq:robust}
\forall \ \e,\delta\in(0,1),\qquad \msf{Q}_r\big(\ms{F}_{n,d}(t), \eta+\e,\delta\big) \leq \frac{2^{Cd^2}}{\eta^{2d}\e}\log\left(\frac{n}{\delta}\right).
\end{equation}
\end{corollary}

It is worth emphasizing that as $t\to0^+$, we can also take $\eta=\eta(t)\to0^+$ in the statement above. Corollary \ref{cor:robust} is a robust version of the main theorem of \cite{EI21}. On the one hand, the method of \cite{EI21} seems unfit to provide estimates for the complexity of $\ms{F}_{n,d}(t)$ as it uses the Bohnenblust--Hille inequality \cite{DMP19}, which heavily relies on the fact that the unknown function is a bounded polynomial. On the other hand (see also Remark \ref{rem:newproof}), Corollary \ref{cor:robust} gives a worse estimate on the complexity of $\ms{F}_{n,d}=\ms{F}_{n,d}(0)$ in terms of $d, \e$ than the bound \eqref{eq:EI} of \cite{EI21}.

To the best of our knowledge, Corollary \ref{cor:robust} is the best known upper bound for the randomized query complexity of the class $\ms{F}_{n,d}(t)$ for $t>0$ after the Low-Degree Algorithm of \cite{LMN93} which gives the  estimate $\msf{Q}_r(\ms{F}_{n,d}(t),t+\e,\delta) \leq \frac{2n^d}{\e} \log\left(\frac{2n^d}{\delta}\right)$. It remains an interesting problem to understand whether one can sharpen the dependence on $d$ of the lower bound for $\eta$ in Corollary \ref{cor:robust}. Specifically for the case of Boolean functions, the implicit dependence of $\eta$ on $t, d$ can be exponentially improved due to an important junta theorem of Bourgain \cite{Bou02}. For $t>0$, denote by $\ms{B}_{n,d}(t)$ the class of all Boolean functions $f:\{-1,1\}^n\to\{-1,1\}$ satisfying \eqref{tail}.

\begin{corollary} \label{cor:robust-boolean}
There exists a universal constant $C>0$ such that the following holds. For any $n\in\N$, $d\in\{1,\ldots,n\}$, $t\in[0,1)$ and $\eta\geq t^{1+o(1)}d^{\frac{1}{2}+o(1)}$, we have\footnote{The explicit nature of the $o(1)$-terms in the exponents will be made precise in Section \ref{sec:3}.}
\begin{equation}
\forall \ \e,\delta\in(0,1),\qquad \msf{Q}_r\big(\ms{B}_{n,d}(t), \eta+\e,\delta\big) \leq \frac{2^{Cd^2}}{\e}\log\left(\frac{n}{\delta}\right).
\end{equation}
\end{corollary}

Finally, we present a concrete application of Corollary \ref{cor:robust-boolean} to Boolean functions which can be represented by constant depth circuits. We refer to \cite[Chapter~4]{O'Do14} for the relevant definitions. For readers which are unfamiliar with this class, we just point out that DNF formulas, i.e.~functions which are representable as logical $\lor$ of terms, each of which is a logical $\land$ of variables $x_i$ or their negations $\lnot x_i$, are circuits of depth 2. Similarly, CNF formulas, in which the roles of $\lor$ and $\land$ are reversed, are also circuits of depth 2. Corollary \ref{cor:robust-boolean} combined with estimates on the Fourier concentration of constant depth circuits \cite{Has87,LMN93,Has01} has the following consequence.

\begin{corollary} \label{cor:circuits}
Let $\ms{C}_{n,d,s}$ be the class of all Boolean functions on $\{-1,1\}^n$ computable by a depth-$d$ circuit of size $s>1$. Then, for every $\e,\delta\in(0,1)$, we have
\begin{equation}
\msf{Q}_r(\ms{C}_{n,d,s},\e,\delta) \leq \exp\big(O(\log (s/\e))^{2(d-2)} \cdot (\log s)^2 \cdot (\log(1/\e))^2\big) \log\left(\frac{n}{\delta}\right).
\end{equation}
\end{corollary}

Learning constant depth circuits (also known as $\msf{AC}^0$ circuits) in quasi-polynomial time is the main focus of the seminal work \cite{LMN93} of Linial, Mansour and Nisan which prompted them to design the Low-Degree Algorithm. Moreover, it is known (see \cite{Kha93}) that quasi-polynomial time is also \emph{necessary} to learn this class, conditionally on some standard cryptographic assumptions. The contribution of Corollary \ref{cor:circuits} is the fact that the query complexity of this learning problem is (exponentially) smaller than the corresponding running time \cite{LMN93} for large enough $n$ (see \cite{KS04} and the references therein for the long-standing open problem of understanding the running time required to learn DNF). It is worth emphasizing that the reason Corollary \ref{cor:circuits} follows from Corollary \ref{cor:robust-boolean} is that $\msf{AC}^0$ circuits have strong enough Fourier concentration. It remains an interesting problem to understand whether Corollary \ref{cor:robust-boolean} can be boosted to encapsulate classes of Boolean functions with weaker concentration such as linear threshold functions or functions of many hyperplanes  \cite{BKS99,Per21}.


\subsection{Metric entropy and learning lower bounds} \label{sec:1.3} Let $(\mathcal{X}, d_\XX)$ be a metric space and $\e>0$. A subset $\mathcal{P} \subseteq \XX$ is an $\e$-packing of $\XX$ if for any $p\neq p'\in\mathcal{P}$, we have $d_\XX(p,p')>\e$. The largest size of an $\e$-packing is called the \emph{packing number} of $\XX$ and is denoted by $\msf{M}(\XX,d_\XX,\e)$. A subset $\mathcal{C}\subseteq\XX$ is an $\e$-cover of $\XX$ if for any $q\in\XX$, there exists some $p\in\mathcal{C}$ with $d_\XX(p,q)\leq \e$. The smallest size of an $\e$-cover is called the \emph{covering number} of $\XX$ and is denoted by $\msf{N}(\XX,d_\XX,\e)$.  The quantity $\log_2\msf{N}(\XX,d_\XX,\e)$ is called the $\e$-\emph{metric entropy} of $\XX$.  It is well known (see \cite[Lemma 4.2.8]{Ver18}) that packing and covering numbers are closely related via the elementary inequalities
\begin{equation}
\forall \ \e>0, \qquad \msf{M}(\XX,d_\XX,2\e)\leq \msf{N}(\XX,d_\XX,\e) \leq \msf{M}(\XX,d_\XX,\e).
\end{equation}

The pertinence of metric entropy in the context of learning lower bounds stems from the classical observation that concept classes with large covering (or packing) numbers cannot be efficiently learned from few queries (see, for instance, the works \cite{BEHW89,LMR91,DKRZ94}). In our setting, we shall need the following concrete estimate which we could not locate in the literature.

\begin{proposition} \label{prop:QM}
Fix $n\in\N$ and let $\ms{B}$ be a class of Boolean functions on $\{-1,1\}^n$. Then,
\begin{equation} \label{det}
\forall \ \e>0, \qquad \msf{Q}(\ms{B},\e) \geq \log_2 \msf{M}(\ms{B},\|\cdot\|_{L_2},2\sqrt{\e})
\end{equation}
and
\begin{equation} \label{rand}
\forall \ \e>0 \mbox{, } \forall \ \delta\in(0,1), \qquad \msf{Q}_r(\ms{B},\e,\delta) \geq \log_2 \msf{M}(\ms{B},\|\cdot\|_{L_2},2\sqrt{\e})+\log_2(1-\delta),
\end{equation}
where $\|\phi-\psi\|_{L_2} = \sqrt{\mb{E}_x (\phi(x)-\psi(x))^2}$ is the $L_2$-norm with respect to the uniform probability measure.
\end{proposition}


The class $\ms{B}_{n,d}$ defined in \eqref{eq:Bnd} contains all Walsh functions $\{w_S\}_{|S|\leq d}$ and thus
\begin{equation} \label{eq:lb-from-walsh}
\forall \ \e \in(0,\sqrt{2}), \qquad \msf{M}(\ms{B}_{n,d},\|\cdot\|_{L_2},\e) \geq \sum_{k=0}^d \binom{n}{k} \geq \frac{n^d}{d^d},
\end{equation}
as $\|w_S-w_T\|_{L_2} = \sqrt{2}$ for any $S\neq T$. Combining this simple lower bound with Proposition \ref{prop:QM} we already deduce the asymptotic sharpness of \eqref{eq:EI} as $n\to\infty$. In order to derive a sharper estimate for $\msf{Q}(\ms{F}_{n,d},\e)$ as a function of the degree $d$, we shall prove the following improved lower bound on the packing numbers of $\ms{B}_{n,d}$ along with a qualitatively matching upper bound for the metric entropy of $\ms{F}_{n,d}$.

\begin{theorem} \label{thm:ent}
Fix $n\in\N$,  $d\in\N$ with $d\leq \log_2n$ and $\e\in(0,1)$. Then, we have
\begin{equation} \label{eq:ent-lb}
\log_2\msf{M}(\ms{B}_{n,d},\|\cdot\|_{L_2},2\e) \geq (1-\e)2^{d-2}\log_2 n - (d+1)2^{d-2}
\end{equation}
Moreover,
\begin{equation} \label{eq:ent-ub}
\log_2 \msf{N}(\ms{F}_{n,d},\|\cdot\|_{L_2},\e) \leq \frac{2^{Cd}}{\e^4} \log n + \kappa(d,\e),
\end{equation}
where $C>0$ is a universal constant and $\kappa(d,\e)>0$ depends only on $d$ and $\e$.
\end{theorem}

\begin{proof} [Proof of Theorem \ref{thm:lb}]
Inequality \eqref{eq:EIS-det} is a direct consequence of \eqref{det}, \eqref{eq:lb-from-walsh} and \eqref{eq:ent-lb}, while \eqref{eq:EIS} follows from \eqref{rand}, \eqref{eq:lb-from-walsh} and \eqref{eq:ent-lb}.
\end{proof}

 Having presented Theorem \ref{thm:ent}, some observations related to the implicit dependencies in \eqref{eq:equiv} are in order. In \cite{EI21} it was shown that
\begin{equation} \label{eq:EI2}
\msf{Q}_r(\ms{F}_{n,d},\e,\delta) \leq \frac{e^8 d^2}{\e^{d+1}} (B_d^{\{\pm 1\}})^{2d} \log\left(\frac{n}{\delta}\right),
\end{equation}
where $B_d^{\{\pm 1\}}$ is an important approximation theoretic parameter called the \emph{Bohnenblust--Hille} constant of the hypercube (see \cite{BH31,DS14}). While it is widely believed that $B_d^{\{\pm 1\}}$ grows at most polynomially in $d$, the best known upper bound due to \cite{DMP19} states that $B_d^{\{\pm 1\}}\leq \exp(C\sqrt{d\log d})$ for some universal constant $C>0$ which, combined with \eqref{eq:EI2},  leads to \eqref{eq:EI}. A polynomial bound on $B_d^{\{\pm1\}}$ combined with \eqref{eq:EI2} would almost match, up to a logarithmic term in the exponent, the asymptotic behavior as $d\to\infty$ of the lower bound \eqref{eq:EIS} of Theorem \ref{thm:lb}. We mention at this point that we are not aware of any non-constant (as $d\to\infty$) lower bound for the constant $B_d^{\{\pm1\}}$.

The existence of a large separated set attaining the lower bound \eqref{eq:ent-lb} is proven via a probabilistic construction of random decision trees with prescribed depth. On the other hand, the upper bound \eqref{eq:ent-ub} is a consequence of the deep junta theorem of  \cite{DFKO07} but, to the extent of our knowledge, had not been previously observed in the literature. It is quite surprising that while $\ms{F}_{n,d}$ lies in a $O_d(n^d)$-dimensional space, its metric entropy is logarithmic in the dimension of this space rather than polynomial. The reason for this is the strong restriction that $\ms{F}_{n,d}$ consists of functions which are bounded in $L_\infty$-norm yet it is endowed with the Hilbertian $L_2$-metric. The existence of such \emph{small nets} is often useful in theoretical computer science and probability theory, in particular in the derandomization literature \cite{Mat96,RS10,KZ20} and in the study of suprema of stochastic processes \cite[Chapters 7-8]{Ver18}. 



\subsection{Exact learning} Having established reasonable bounds on the number of queries required to learn a function in $\ms{F}_{n,d}$ up to error $\e>0$, we proceed to investigate the exact case $\e=0$. As it turns out, the number of random queries required to learn a function $f\in\ms{F}_{n,d}$ up to a constant error $\e\in(0,1)$ using the classical Low-Degree Algorithm \cite{LMN93} is in fact the same (up to constants depending only on $d$) as the number of queries required to \emph{exactly} learn the concept class $\ms{F}_{n,d}$. Formally, our results is this setting are summarized in the following theorem.

\begin{theorem} \label{thm:exact}
Fix $n\in\N$,  $d\in\{1,\ldots,n\}$ and $\delta\in(0,1)$. Then,
\begin{equation} \label{eq:exact}
\msf{Q}(\ms{F}_{n,d},0) = \sum_{j=0}^d \binom{n}{j}
\end{equation}
and there exists a universal constant $C>0$ such that
\begin{equation} \label{eq:e=0}
\msf{Q}_r(\ms{F}_{n,d},0,\delta) \leq Cd2^d n^d \log\left(\frac{n}{\delta}\right).
\end{equation}
\end{theorem}


\begin{remark}
Throughout this paper and \cite{EI21}, we study learning algorithms for $\ms{F}_{n,d}$ and $\ms{B}_{n,d}$ equipped with the Hilbertian $L_2$-metric. 
This choice allows us to use Parseval's identity and thus exploit properties of individual Walsh coefficients to study the distance between $f$ and the hypothesis function $h$. 
However as the constructed hypothesis functions $h$ are always of degree at most $d$ themselves, this can be generalized to any $L_p$ norm, where $0<p<\infty$, since these are equivalent to the $L_2$ norm on the space of degree-$d$ polynomials up to constants depending only on $d$ (see \cite[\S 9.5]{O'Do14} and \cite{Bou80, Bor84, EI20} for more on moment comparison of polynomials).
\end{remark}


\subsection*{Structure of the paper} In Section \ref{sec:2}, we prove our main lower bounds for learning, namely Proposition \ref{prop:QM} and Theorem \ref{thm:ent}. In Section \ref{sec:3}, we prove Theorem \ref{thm:upper} and deduce from it the Corollaries of Section \ref{sec:upper}. Finally, in Section \ref{sec:4}, we prove Theorem \ref{thm:exact} on exact learning.


\section{Metric entropy and query complexity} \label{sec:2}

We start by formalizing the concepts introduced earlier. A \emph{learning algorithm} on $\{-1,1\}^n$ using $Q$ queries is a mapping $H:\left(\{-1,1\}^n\times\R\right)^Q\to L_2(\{-1,1\}^n)$ which, given input of the form $(X_1,f(X_1)), \ldots, (X_Q,f(X_Q))$ produces a hypothesis function for $f$. In this terminology,  the randomized query complexity $\msf{Q}_r(\ms{F},\e,\delta)$ of a class of functions $\ms{F}$ on the hypercube is the smallest $Q\in\N$ for which there exists a learning algorithm $H$ with the following property:
\begin{equation} \label{eq:dQr}
\forall \ f\in\ms{F},\qquad \mb{P}_{X_1,\ldots,X_Q\in\{-1,1\}^n} \Big\{\big\| H\big((X_1,f(X_1)),\ldots,(X_Q,f(X_Q))\big) - f\big\|_{L_2}^2 \leq \e\Big\} \geq 1-\delta.
\end{equation}
In the case of non-randomized algorithms,  we need to ensure that the query points are chosen consistently with respect to the previous data.  In other words,  $X_1$ is always a fixed point on the hypercube and for any $q\geq2$,  there exists a function $\varphi_q:(\{-1,1\}^n\times\R)^{q-1}\to\{-1,1\}^n$ associated to $H$ determining the $q$-th query point as a function of the previous data $X_1,\ldots,X_{q-1}$ and the values $y_1,\ldots,y_{q-1}$ of the unknown function on these points.  Given a learning algorithm $H$ and an unknown function $f\in\ms{F}$, we shall denote by $X_1[f], X_2[f],\ldots$ the sequence of points that $H$ queries in order to construct a hypothesis function for $f$. In this terminology,  the deterministic query complexity of the class $\ms{F}$ is the least $Q\in\N$ for which there exists a learning algorithm $H$ using $Q$ queries satisfying the following property:
\begin{equation} \label{eq:dQ}
\forall \ f\in\ms{F}, \qquad  \big\| H\big(\big(X_1[f],f(X_1[f])\big),\ldots,\big(X_Q[f],f(X_Q[f])\big)\big) - f\big\|_{L_2}^2 \leq \e.
\end{equation}

Having properly defined these notions, we may proceed to the proof of Proposition \ref{prop:QM}. The argument relies on an information-theoretic consideration: given samples $X_1,\ldots,X_Q$, the outputs $f(X_1),\ldots,f(X_Q)$ provide $Q$ bits of information for $f$ and thus cannot distinguish more than $\log_2Q$ functions  which are reasonably far apart.

\begin{proof} [Proof of Proposition \ref{prop:QM}]
Let $M=\msf{M}(\ms{B},\|\cdot\|_{L_2},2\sqrt{\e})$ and consider $f_1,\ldots,f_M\in\ms{B}$ with $\|f_i-f_j\|_{L_2} > 2\sqrt{\e}$ for all $i\neq j$.  We start with the lower bound \eqref{det} in the deterministic case.  Denote by $Q=\msf{Q}(\ms{B},\e)$ and let $X_1[f],X_2[f],\ldots,X_Q[f]$ be samples satisfying \eqref{eq:dQ} for some learning algorithm $H$ and all functions $f$ in the class $\ms{B}$. Consider the set
\begin{equation}
\Sigma \eqdef \big\{ \big( f_i(X_1[f_i]),\ldots,f_i(X_Q[f_i]) \big) : \ i=1,\ldots,M\big\}.
\end{equation}
\noindent {\it Claim.} $|\Sigma| = M$.
\smallskip

\noindent {\it Proof.} Clearly $|\Sigma|\leq M$. If $|\Sigma|<M$, then there exist $i\neq j\in\{1,\ldots,M\}$ for which we have 
\begin{equation} \label{eq:coincide}
\forall \ k\in\{1,\ldots,Q\},\qquad f_i(X_k[f_i]) = f_j(X_k[f_j]).
\end{equation}
As $X_1[f_i]=X_1[f_j]\eqdef X_1$ by definition of $H$,  \eqref{eq:coincide} gives $f_i(X_1)=f_j(X_1)$ which then, by consistency of the algorithm, implies that $X_2[f_i]=X_2[f_j]\eqdef X_2$. Continuing iteratively, we deduce that $X_k[f_i]=X_k[f_j]\eqdef X_k$ for every $k\in\{1,\ldots,Q\}$ and thus the common output function
\begin{equation}
h \eqdef  H\big((X_1,f_i(X_1)),\ldots,(X_Q,f_i(X_Q))\big) =  H\big((X_1,f_j(X_1)),\ldots,(X_Q,f_j(X_Q))\big)
\end{equation}
satisfies $\|h-f_i\|_{L_2}^2\leq \e$ and $\|h-f_j\|_{L_2}^2 \leq \e$ which is a contradiction as $\|f_i-f_j\|_{L_2} > 2\sqrt{\e}$. \hfill$\Box$

\smallskip

Finally, observe that as the class $\ms{B}$ consists of Boolean functions, we have the trivial inclusion $\Sigma\subseteq \{-1,1\}^Q$ which implies that $M=|\Sigma|\leq 2^Q$ and the proof is complete.

\smallskip

In the random case, denote by $Q=\msf{Q}_r(\ms{F},\e,\delta)$ and let  $X=(X_1,\ldots,X_Q)$ where $X_1,X_2,\ldots$ are independent random vectors, each uniformly distributed on $\{-1,1\}^n$, satisfying \eqref{eq:dQr} for some learning algorithm $H$. For every $i\in\{1,\ldots,M\}$, consider the event 
\begin{equation}
B_i\eqdef \Big\{\big\| H\big((X_1,f_i(X_1)),\ldots,(X_Q,f_i(X_Q))\big) - f_i\big\|_{L_2}^2 > \e\Big\},
\end{equation}
which has probability $\mb{P}\{B_i\}\leq\delta$ by \eqref{eq:dQr} and, as before, consider the (random) set
\begin{equation}
\Sigma(X) \eqdef \big\{ \big( f_i(X_1),\ldots,f_i(X_Q) \big): \ i=1,\ldots,M\big\}.
\end{equation}
\noindent {\it Claim.} $\mb{E}|\Sigma(X)| \geq (1-\delta)M$.
\smallskip

\noindent {\it Proof.} Consider the partition $\{1,\ldots,M\} = \sigma_1 \sqcup \cdots \sqcup \sigma_{|\Sigma(X)|}$ depending on $X$ such that for every $r\in\{1,\ldots,|\Sigma(X)|\}$ and all $i,j \in\sigma_r$, we have $f_i\equiv f_j$ on $\{X_1,\ldots,X_Q\}$.  Now, suppose that there exist two distinct $i\neq j\in\sigma_r$ such that $X\notin B_i$ and $X\notin B_j$. Then, the function
\begin{equation}
h \eqdef  H\big((X_1,f_i(X_1)),\ldots,(X_Q,f_i(X_Q))\big) =  H\big((X_1,f_j(X_1)),\ldots,(X_Q,f_j(X_Q))\big)
\end{equation}
satisfies $\|h-f_i\|_{L_2}^2\leq\e$ and $\|h-f_j\|_{L_2}^2\leq\e$ which contradicts $\|f_i-f_j\|_{L_2}>2\sqrt{\e}$.  Therefore,  for any $r$ and any $X=(X_1,\ldots,X_Q)$, there exists a subset $\tau_r\subseteq\sigma_r$ with $|\tau_r|\geq|\sigma_r|-1$ such that $X\in B_i$ for all $i\in\tau_r$.  Adding up these inequalities and taking the expectation, we deduce that
\begin{equation}
M-\mb{E}|\Sigma(X)| = \mb{E}\Big[ \sum_{r=1}^{|\Sigma(X)|} \big(|\sigma_r|-1\big) \Big] \leq \mb{E}\Big[ \sum_{r=1}^{|\Sigma(X)|} |\tau_r| \Big] \leq \mb{E}\Big[ \sum_{r=1}^{|\Sigma(X)|} \sum_{i\in\sigma_r} {\bf 1}_{B_i}(X) \big] = \sum_{i=1}^{M} \mb{P}\{B_i\} \leq \delta M,
\end{equation}
which is the desired inequality. \hfill$\Box$

\smallskip

 As the class $\ms{B}$ consists of Boolean functions, we have $(1-\delta)M\leq\mb{E}|\Sigma(X)|\leq 2^Q$.
\end{proof}


\subsection{Decision trees} In this section we will prove the lower bound \eqref{eq:ent-lb} for the packing numbers of $\ms{B}_{n,d}$ and $\ms{F}_{n,d}$. First, we introduce some necessary background. Following \cite[\S 3.2]{O'Do14}, we define a decision tree $T$ to be a representation of a function $f:\{-1,1\}^n\to\R$ as a rooted binary tree in which the internal nodes are labeled by Boolean variables $x_i$, $i\in\{1,\ldots,n\}$, the edges are labeled by -1 and 1 and the leaves are labeled by real numbers. It is required that no Boolean variable $x_i$ appears more than once on any root-leaf path. On input $y\in\{-1,1\}^n$, the tree $T$ computes the value $f(y)$ in the following way. Starting from the root, when the computation path reaches a node labeled by $x_i$, it follows the unique edge labeled by the value $y_i\in\{-1,1\}$. The output $f(y)$ of $T$ is the label of the leaf reached by this path. It is a classical fact (see \cite[Proposition~3.16]{O'Do14}) that if a function $f$ can be represented by a decision tree of depth $d$, then $f$ has degree at most $d$.

In order to prove the lower bound \eqref{eq:ent-lb} on the packing number of $\ms{B}_{n,d}$, we shall need the following combinatorial lemma on large families of sets with pairwise small intersections.

\begin{lemma} \label{aux}
Fix $m,k\in\N$ with $k<m$ and $\e\in(0,1)$.  Then, there exists $t\geq (2k)^{-k/2} m^{(1-\e)k/2}$ and subsets $\sigma_1,\ldots, \sigma_t \subset \{1,\ldots,m\}$ of size $k$ satisfying
\begin{equation} \label{eq:intersection}
\forall \ i\neq j\in\{1,\ldots,t\},\qquad |\sigma_i\cap \sigma_j| < (1-\e)k.
\end{equation}
\end{lemma}

\begin{proof}
We shall use the probabilistic method. Suppose that $\boldsymbol{\sigma}$ is a uniformly chosen random subset of $\{1,\ldots,m\}$ of cardinality $k$. Then, we have
\begin{equation}
\mb{P}\big\{ |\boldsymbol{\sigma}\cap\{1,\ldots,k\}| \geq (1-\e)k \big\}=\frac{1}{\binom{m}{k}} \sum_{j\leq\e k} \binom{k}{k-j} \binom{m-k}{j} \leq \frac{k^k}{m^k} m^{\e k} \sum_{j\leq\e k} \binom{k}{k-j} \leq (2k)^k m^{-(1-\e)k},
\end{equation}
where we used the fact that $\frac{r^s}{s^s} \leq \binom{r}{s} \leq r^s$. If $\boldsymbol{\sigma}_1, \ldots, \boldsymbol{\sigma}_t$ are i.i.d.~copies of $\boldsymbol{\sigma}$, then by homogeneity
\begin{equation}
\forall \ i\neq j\in\{1,\ldots,t\}, \qquad \PP\big\{|\boldsymbol{\sigma}_i \cap \boldsymbol{\sigma}_j|\geq(1-\e) k\big\} \leq (2k)^k m^{-(1-\e)k}
\end{equation}
and thus  
\begin{equation}
\E \left[  \#\{i \neq j: |\boldsymbol{\sigma}_i\cap \boldsymbol{\sigma}_j|\geq(1-\e) k \} \right] \leq \binom{t}{2} (2k)^k m^{-(1-\e)k} < t^2(2k)^k m^{-(1-\e)k}.
\end{equation}
Therefore, if $t\leq (2k)^{-k/2} m^{(1-\e)k/2}$, there exist $\sigma_1,\ldots,\sigma_t$ with the desired property.
\end{proof}


\tikzstyle{level 1}=[level distance=2cm, sibling distance=3.5cm]
\tikzstyle{level 2}=[level distance=2cm, sibling distance=3.5cm]
\tikzstyle{level 3}=[level distance=2cm, sibling distance=3.5cm]
\tikzstyle{level 4}=[level distance=1.7cm, sibling distance=2.5cm]
\tikzstyle{bag} = [text width=3.5em, text centered]
\tikzstyle{end} = [circle, minimum width=3pt,fill, inner sep=0pt]

\begin{figure}
\begin{center}
\begin{tikzpicture}[grow=down,
emph/.style={edge from parent/.style={black, dashed,thin,draw}},
    norm/.style={edge from parent/.style={solid, black,thin,draw}},
    transp/.style={edge from parent/.style={transparent}}]
\node[bag] {\tiny$\circled{\normalsize$x_1$}$}
    child {
        node[bag] {\tiny$\circled{\normalsize$x_2$}$}        
			child[emph] {
				node[bag]  {\tiny$\circled{\normalsize$x_{d-1}$}$}         
					child[norm] {	
						node[bag]  {\tiny$\circled{\normalsize$x_{i_1}$}$}  
							child{  		
                					node[end, label=below:
                    					{$-1$}] {r}
									edge from parent                
               						 node[left] {$-1$}			               
            							}
            					child {
                					node[end, label=below:
                   				 {$1$}] {r}
                					edge from parent
                					node[right] {$1$}
            						}
            					edge from parent
            					node[left]{$-1$}   
            			}
            			child[norm] {	
						node[bag]  {\tiny$\circled{\normalsize$x_{i_2}$}$}  
							child{  		
                					node[end, label=below:
                    					{$-1$}] {r}
									edge from parent                
               						 node[left] {$-1$}			               
            							}
            					child {
                					node[end, label=below:
                   				 {$1$}] {r}
                					edge from parent
                					node[right] {$1$}
            						} 
            						edge from parent
            					node[right]{$1$} 
            			}
            }         
            child[emph] {
                node[bag] {$\vdots$}
                child[emph] {
                		node[bag] {\tiny$\circled{\normalsize$x_{i_2}$}$}
                		}
                		child[emph] {
                		node[bag] {\tiny$\circled{\normalsize$x_{i_{k-1}}$}$}
                		}
                edge from parent
            }
            edge from parent
            node[left] {$-1$} 
    }
    child {
        node[bag]  {\tiny$\circled{\normalsize$x_2$}$}  
         child[emph] {
                node[bag] {$\vdots$}
                		child[transp] {
                		node[bag] {}
                		}
                		child[transp] {
                		node[bag] {}
                		}
                edge from parent
            }        
			child[emph] {
				node[bag] {\tiny$\circled{\normalsize$x_{d-1}$}$}         
					child[norm] {	
						node[bag]  {\tiny$\circled{\normalsize$x_{i_{k-1}}$}$}  
							child{  		
                					node[end, label=below:
                    					{$-1$}] {r}
									edge from parent                
               						 node[left] {$-1$}			               
            							}
            					child {
                					node[end, label=below:
                   				 {$1$}] {r}
                					edge from parent
                					node[right] {$1$}
            						}   
            					edge from parent
            					node[left] {$-1$} 
            			}
            			child[norm] {	
						node[bag]  {\tiny$\circled{\normalsize$x_{i_{k}}$}$}  
							child{  		
                					node[end, label=below:
                    					{$-1$}] {r}
									edge from parent                
               						 node[left] {$-1$}			               
            							}
            					child {
                					node[end, label=below:
                   				 {$1$}] {r}
                					edge from parent
                					node[right] {$1$}
            						}   
            					edge from parent
            					node[right]{$1$}
            			}
            } 
            edge from parent
            node[right] {$1$}        
    };
\end{tikzpicture}
\captionsetup{labelformat=empty}
\caption{The decision tree $T_\sigma$ corresponding to $\sigma=\{i_1<i_2<\cdots<i_k\}$}
\end{center}
\end{figure}
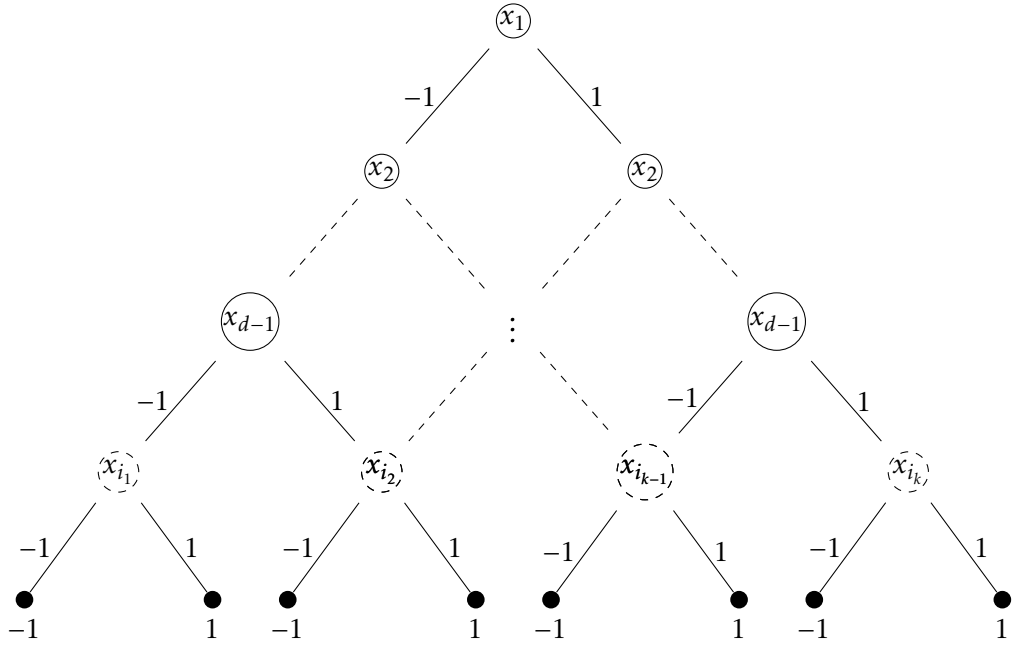


Equipped with Lemma \ref{aux}, we proceed to the proof of the lower bound in Theorem \ref{thm:ent}.

\begin{proof} [Proof of \eqref{eq:ent-lb}]
Let $\sigma$ be a subset of $\{d,d+1,\ldots,n\}$ of cardinality $k=2^{d-1}$. We shall associate to $\sigma$ a Boolean function $f_\sigma:\{-1,1\}^n\to\{-1,1\}$ of degree at most $d$ represented by the decision tree $T_\sigma$ which is constructed as follows. The root of $T_\sigma$ is labeled by $x_1$ and every node which is at distance $i$ from the root is labeled by $x_{i+1}$ for $i\in\{1,\ldots,d-2\}$. If $\sigma=\{i_1,\ldots,i_{k}\}$, then the nodes at distance $d-1$ from the root are labeled by the distinct variables $x_{i_1},\ldots,x_{i_k}$ in accordance with the lexicographic ordering $\leq_L$, meaning that if $(\e_r(1),\ldots,\e_r(d-1)), (\e_s(1),\ldots,\e_s(d-1))\in\{-1,1\}^{d-1}$ are the labels of the edges joining the root with the nodes labeled by $x_{i_r}$ and $x_{i_s}$, then
\begin{equation} \label{eq:ordering}
i_r \leq i_s \quad \Longleftrightarrow \quad (\e_r(1),\ldots,\e_r(d-1)) \leq_L (\e_s(1),\ldots,\e_s(d-1)).
\end{equation}
Finally, if given an input $y\in\{-1,1\}^n$ the tree $T_\sigma$ queries the variable $x_{i_j}$ on the $d$-th level, then its output is $y_{i_j}$. This construction is depicted pictorially in the figure above. Observe that in this picture, the restriction \eqref{eq:ordering} is equivalent to $i_1<i_2<\ldots<i_k$.

Using Lemma \ref{aux}, we can find $t\geq 2^{-d2^{d-2}} (n-d+1)^{(1-\e)2^{d-2}}\geq 2^{-(d+1)2^{d-2}}n^{(1-\e)2^{d-2}}$ and subsets $\sigma_1,\ldots,\sigma_t$ of $\{d,d+1,\ldots,n\}$ with cardinality $k$ satisfying \eqref{eq:intersection}. We will show that the family of functions $f_{\sigma_1},\ldots,f_{\sigma_t}\in\ms{B}_{n,d}$ is well-separated. Indeed, let $r\neq s$ and suppose that $\sigma_r=\{i_1,\ldots,i_k\}$ and $\sigma_s=\{j_1,\ldots,j_k\}$ with $i_1<\ldots<i_k$ and $j_1<\ldots<j_k$. Then, we have
\begin{equation}
\|f_{\sigma_r}-f_{\sigma_s}\|_{L_2}^2 = \frac{1}{2^{d-1}} \sum_{\ell=1}^k \mb{E}(x_{i_\ell}-x_{j_\ell})^2 = \frac{1}{2^{d-2}} |\{\ell: \ i_\ell\neq j_\ell\}| \geq \frac{2^{d-1} - |\sigma_r\cap\sigma_s|}{2^{d-2}}\stackrel{\eqref{eq:intersection}}{\geq} 2\e
\end{equation}
and the proof is complete.
\end{proof}

While Theorem \ref{thm:lb} provides a logarithmic lower bound for the query complexity of learning $\ms{F}_{n,d}$ in both the query and the random example models, the upper bound \eqref{eq:EI} of \cite{EI21} is currently known to hold only in the random case. Derandomizing the algorithm used there or finding a different deterministic algorithm whose query complexity is logarithmic in the dimension (and which, ideally,  has reasonable running time) remains an interesting problem.


\subsection{Juntas} In this section we will prove the upper bound \eqref{eq:ent-ub} for the metric entropy of the class $\ms{F}_{n,d}$. A general principle in analysis on the hypercube asserts that functions whose spectrum is not spread out, effectively depend only on few variables. Many concrete instances of this phenomenon have been studied for Boolean functions, such as the important works \cite{Fri99, Bou02, FKN02, KS02}. The definitive junta theorem for general \emph{bounded} functions, is the following deep result of Dinur, Friedgut, Kindler and O'Donnell (DFKO) \cite{DFKO07} (see also \cite{OZ16} for a quantitatively sharp statement in terms of the dependence on $d$).

\begin{theorem} \label{thm:dfko}
Fix $n,d\in\N$, $\e>0$ and let $f:\{-1,1\}^n\to[-1,1]$ be a function satisfying
\begin{equation}
\sum_{|S|>d} \hat{f}(S)^2 \leq \exp\big(-C(d^2 \log d)/\e^2\big)
\end{equation} 
for a large enough universal constant $C>0$. Then, there exists a subset $\sigma\subseteq\{1,\ldots,n\}$ with $|\sigma|\leq \tfrac{2^{Cd}}{\e^4}$ and a function $g:\{-1,1\}^n\to\R$ depending only on the variables $(x_i)_{i\in\sigma}$ such that $\|f-g\|_{L_2}\leq \e$.
\end{theorem}

\begin{proof} [Proof of \eqref{eq:ent-ub}]
Let $\msf{m}_{d,\e}$ be the size of the smallest $\tfrac{\e}{4}$-net on the space of all bounded functions $h:\{-1,1\}^{k_{d,\e}}\to[-1,1]$, where $k_{d,\e} = \tfrac{2^{Cd+2}}{\e^4}$, equipped with the $L_2$-metric and let $\{h_1,\ldots,h_{\msf{m}_{d,\e}}\}$ be such a net. For a subset $\sigma \subseteq \{1,\ldots,n\}$ of cardinality $k_{d,\e}$ and $s\in\{1,\ldots,\msf{m}_{d,\e}\}$, define
\begin{equation}
\forall \ x\in\{-1,1\}^n,\qquad h_s^\sigma(x) \eqdef h_s\big( (x_i)_{i\in\sigma}\big).
\end{equation}
\noindent {\it Claim.} The set $\{h_s^\sigma: \ s=1,\ldots,\msf{m}_{d,\e} \ \mbox{and} \ \sigma\subseteq\{1,\ldots,n\} \mbox{ with } |\sigma|=k_{d,\e} \}$ is an $\frac{\e}{2}$-covering of $\ms{F}_{n,d}$.
\smallskip

\noindent {\it Proof.} Indeed, let $f:\{-1,1\}^n\to[-1,1]$ be a function of degree at most $d$. By Theorem \ref{thm:dfko}, there exists a subset $\sigma\subseteq\{1,\ldots,n\}$ with $|\sigma|\leq k_{d,\e}$ and a function $g:\{-1,1\}^n\to\R$ depending only on the variables $(x_i)_{i\in\sigma}$ such that $\|f-g\|_{L_2}\leq \frac{\e}{4}$. Notice that without loss of generality we can assume that $g$ takes values in $[-1,1]$ as we can otherwise define $\tilde{g}:\{-1,1\}^n\to[-1,1]$ by
\begin{equation}
\forall \ x\in\{-1,1\}^n,\qquad \tilde{g}(x) = \begin{cases} g(x), & g(x)\in[-1,1] \\ \mathrm{sign}(g(x)), & g(x)\notin[-1,1]  \end{cases}
\end{equation}
and observe that $\|\tilde{g}-f\|_{L_2}\leq \|g-f\|_{L_2}\leq \tfrac{\e}{4}$. Therefore, by  definition of the covering $\{h_1,\ldots,h_{\msf{m}_{d,\e}}\}$ there exists $s\in\{1,\ldots,\msf{m}_{d,\e}\}$ such that $\|g-h_{s}^\sigma\|_{L_2} \leq \frac{\e}{4}$ and hence $\|f-h_s^\sigma\|_{L_2} \leq  \frac{\e}{2}$. \hfill$\Box$

\smallskip

To conclude, for every $s$ and $\sigma$, choose an arbitrary point $p_s^\sigma\in \ms{F}_{n,d}$ satisfying $\|p_s^\sigma- h_s^\sigma\|_{L_2} \leq \frac{\e}{2}$, provided that such exists (in the opposite case the corresponding ball can be omitted from the cover). Then,
\begin{equation}
\ms{F}_{n,d} \subseteq \bigcup_{s, \sigma} \mathrm{Ball}(h_s^\sigma,\tfrac{\e}{2}) \subseteq \bigcup_{s,\sigma} \mathrm{Ball}(p_s^\sigma, \e),
\end{equation}
thus proving that $\msf{N}(\ms{F}_{n,d},\|\cdot\|_{L_2},\e) \leq \msf{m}_{d,\e} \binom{n}{k_{d,\e}} \leq \msf{m}_{d,\e} n^{\frac{2^{Cd+2}}{\e^4}}$. This concludes the proof.
\end{proof}


\section{Fourier concentration and learning} \label{sec:3}

To prove Theorem \ref{thm:upper}, we shall employ a modification of the algorithms of \cite{LMN93,KM93} with one important twist from the analysis of \cite{EI21}.  We include the argument in full detail for completeness.

\begin{proof} [Proof of Theorem \ref{thm:upper}]
Fix a parameter $b\in(0,\infty)$ to be determined later and denote by
\begin{equation} \label{eq:defQ}
Q_b \eqdef \left\lceil \frac{2}{b^2} \log\left(\frac{2}{\delta}\sum_{r=0}^{d} \binom{n}{r}\right) \right\rceil.
\end{equation}
Let $X_1,\ldots,X_{Q_b}$ be independent random vectors, each uniformly distributed on $\{-1,1\}^n$. For a subset $S\subseteq\{1,\ldots,n\}$ with $|S|\leq d$, consider the empirical Walsh coefficient of $f$ given by
\begin{equation}
\alpha_S = \frac{1}{Q_b} \sum_{j=1}^{Q_b} f(X_j) w_S(X_j).
\end{equation}
As $\alpha_S$ is a sum of bounded i.i.d.~random variables and $\mb{E}[\alpha_S]=\hat{f}(S)$, the Chernoff bound gives
\begin{equation}
\forall \ S\subseteq\{1,\ldots,n\}\mbox{ with } |S|\leq d, \qquad \mb{P}\big\{ |\alpha_S-\hat{f}(S)| > b\big\} \leq 2\exp(-Q_bb^2/2).
\end{equation}
Therefore, using the union bound, we get
\begin{equation*} \label{eq:defG}
\mb{P}\underbrace{\big\{ |\alpha_S-\hat{f}(S)| \leq b, \ \mbox{for every subset } S  \ \mbox{with} \ |S|\leq  d\big\}}_{G_b} \geq 1-2\sum_{r=0}^{d} \binom{n}{r} \exp(-Q_bb^2/2) \stackrel{\eqref{eq:defQ}}{\geq} 1-\delta.
\end{equation*}
Consider the random collection of sets given by
\begin{equation} \label{eq:defS}
\ms{T}_{b} \eqdef \big\{ S\subseteq\{1,\ldots,n\}: \ |S|\leq d\mbox{ and } |\alpha_S| \geq2b\big\}.
\end{equation}
Observe that if the event $G_b$ holds, then
\begin{equation} \label{eq:notinS}
\forall \ S\notin\ms{T}_{b} \mbox{ with }|S|\leq d, \qquad |\hat{f}(S)| \leq |\alpha_S-\hat{f}(S)| + |\alpha_S| \leq 3b
\end{equation}
and
\begin{equation} \label{eq:inS}
\forall \ S\in\ms{T}_{b}, \qquad |\hat{f}(S)| \geq |\alpha_S| - |\alpha_S-\hat{f}(S)| \geq b.
\end{equation}
Now, consider the random function \mbox{$h_{b}:\{-1,1\}^n\to\R$}, given by
\begin{equation}
\forall \ x\in\{-1,1\}^n, \qquad h_{b}(x) \eqdef \sum_{S\in\ms{T}_{b}} \alpha_S w_S(x)
\end{equation}
and write
\begin{equation*}
\begin{split}
\|&h_{b}-f\|_{L_2}^2   = \sum_{S\subseteq\{1,\ldots,n\}} \big|\hat{h}_{b}(S)-\hat{f}(S)\big|^2 = \sum_{S\in\ms{T}_{b}} |\alpha_S-\hat{f}(S)|^2 + \sum_{S\notin\ms{T}_{b}} |\hat{f}(S)|^2 
\\ & = \sum_{\substack{S\in\ms{T}_{b}\\ S\in\ms{S}(f)}} |\alpha_S-\hat{f}(S)|^2 + \sum_{\substack{S\in\ms{T}_{b}\\ S\notin\ms{S}(f)}} |\alpha_S-\hat{f}(S)|^2+\!\!\!\!\!\sum_{\substack{S\notin\ms{T}_b\\S\in\ms{S}(f),\ |S|\leq d}}|\hat{f}(S)|^2  +\!\!\!\!\!\sum_{\substack{S\notin\ms{T}_b\\S\notin\ms{S}(f), \ |S|\leq d}} |\hat{f}(S)|^2 + \sum_{|S|>d} \hat{f}(S)^2.
\end{split}
\end{equation*}
On the event $G_b$ we then have
\begin{equation} \label{will-use-BH}
 \sum_{\substack{S\in\ms{T}_{b}\\ S\in\ms{S}(f)}} |\alpha_S-\hat{f}(S)|^2 + \!\!\!\!\!\sum_{\substack{S\notin\ms{T}_b\\S\in\ms{S}(f),\ |S|\leq d}}|\hat{f}(S)|^2 \stackrel{\eqref{eq:notinS}}{\leq} (9b^2)\cdot\#\ms{S}(f) \leq 9b^2m.
\end{equation}
On the other hand, as $|\alpha_S-\hat{f}(S)|\leq b \leq |\hat{f}(S)|$ for $S\in\ms{T}_b$, we get
\begin{equation}
\sum_{\substack{S\in\ms{T}_{b}\\ S\notin\ms{S}(f)}} |\alpha_S-\hat{f}(S)|^2+\!\!\!\!\!\sum_{\substack{S\notin\ms{T}_b\\S\notin\ms{S}(f), \ |S|\leq d}} |\hat{f}(S)|^2 \stackrel{\eqref{eq:inS}}{\leq} \sum_{S\notin\ms{S}(f)} |\hat{f}(S)|^2\leq  \eta
\end{equation}
by the Fourier concentration property. Combining the above with the assumption that the spectrum of $f$ is $t$-concentrated up to degree $d$, we conclude that
\begin{equation}
\|h_b-f\|_{L_2}^2 \leq \eta + t + 9b^2m \leq \eta + t + \e
\end{equation}
for $b^2 \leq \e/9m$. Plugging this choice of $b$ in \eqref{eq:defQ}, we get the conclusion.
\end{proof}

We are now well-equipped to prove Corollary \ref{thm:junta}.

\begin{proof} [Proof of Corollary \ref{thm:junta}]
Let $f\in\ms{F}_{n,d}\cap\ms{J}_{n,k,\eta}$. Then, there exists a subset $\sigma\subseteq\{1,\ldots,n\}$ with $|\sigma|\leq k$ and a function $g:\{-1,1\}^n\to\R$ depending only on the variables $(x_i)_{i\in\sigma}$ such that $\|f-g\|_{L_2}^2\leq\eta$. Then, $f$ is $\eta$-concentrated on the collection $\ms{S}(f) = \{S\subseteq\sigma: \ |S|\leq d\}$, as
\begin{equation}
\sum_{S\nsubseteq\sigma} \hat{f}(S)^2 \leq \|f-g\|_{L_2}^2 \leq \eta.
\end{equation}
Similarly,  the spectrum of $f$ is $\eta$-concentrated up to degree $\min\{d,k\}$ and the conclusion of the corollary follows from Theorem \ref{thm:upper} since $\#\ms{S}(f) = \sum_{r=0}^{\min\{d,k\}} \binom{k}{r}$.
\end{proof}

We emphasize that Corollary \ref{thm:junta} does not make any claim about the \emph{running time} required to learn approximate juntas. This is a notoriously difficult problem even for actual juntas that has been investigated in a series of important works, see for instance \cite{MOS03,KLMMV09,ST14}.


Corollary \ref{prop:boolean} follows from Corollary \ref{thm:junta} combined with a classical theorem of Nisan and Szegedy \cite{NS94}, asserting that a Boolean function of degree $d$ depends on at most $d2^{d-1}$ variables. We note in passing that this result has recently been improved in important work of Chiarelli, Hatami and Saks \cite{CHS20} (see also \cite{Wel19} for the best known value of the implicit constant) who derived the optimal conclusion that such a function only depends on $O(2^d)$ variables, but this refinement will be immaterial for our considerations.

\begin{proof} [Proof of Corollary \ref{prop:boolean}]
By \cite[Theorem~1.2]{NS94}, we have the set inclusion $\ms{B}_{n,d} \subseteq \ms{F}_{n,d}\cap\ms{J}_{n,k,0}$ where \mbox{$k=d2^{d-1}$}. Therefore, by Corollary \ref{thm:junta}, we conclude that
\begin{equation}
\msf{Q}_r(\ms{B}_{n,d},\e,\delta) \leq \frac{18}{\e} \sum_{r=0}^d \binom{d2^{d-1}}{r} \log\left( \frac{2}{\delta} \sum_{r=0}^d \binom{n}{r}\right) \leq \frac{36\cdot d2^{d^2}}{\e} \log\left(\frac{n}{\delta}\right)
\end{equation}
where the last inequality follows by elementary estimates.
\end{proof}

We now proceed to prove Corollary \ref{cor:robust}, which relies on Theorem \ref{thm:upper} and \cite{DFKO07}.

\begin{proof} [Proof of Corollary \ref{cor:robust}]
Let $f\in\ms{F}_{n,d}(t)$ and $\eta\geq\frac{C d^2{\log d}}{{\log(1/t)}}$ so that
\begin{equation}
\sum_{|S|>d} \hat{f}(S)^2 \leq t \leq \exp\big(-C(d^2\log d)/\eta\big).
\end{equation}
Instead of using Theorem \ref{thm:dfko} directly, we will use a stronger statement from its proof.  In \cite[p.~405]{DFKO07}, it was shown that there exists a function $h$ of degree at most $d$ which depends only on the variables $(x_i)_{i\in\sigma}$ for a subset $\sigma\subseteq\{1,\ldots,n\}$ with $|\sigma| \leq 2^{O(d)}/\eta^2$ such that $\|f-h\|_{L_2}^2\leq\eta$.  Choosing $\ms{S}(f)=\{S\subseteq\sigma: \ |S|\leq d\}$, we deduce that $f$ is $\eta$-concentrated up to degree $d$ and on the collection $\ms{S}(f)$. The conclusion follows from Theorem \ref{thm:upper} as $\#\ms{S}(f)\leq 2^{O(d^2)}/{\eta^{2d}}$.
\end{proof}

We note in passing that one can replace the use of the DFKO theorem with a result of O'Donnell and Zhao \cite[Corollary~3.5]{OZ16} to improve the dependence of $\eta$ on $d$ to $\eta\geq \frac{Cd^2}{{\log(1/t)}}$ at the expense of an exponentially worse dependence of the complexity on $d$ and $\e$.

\begin{remark} \label{rem:newproof}
Choosing $t=0$ in Corollary \ref{cor:robust} provides a different proof of the main result of \cite{EI21}, i.e.~that $\msf{Q}_r(\ms{F}_{n,d},\e,\delta)  = O_{d,\e,\delta}(\log n)$, using the DFKO theorem. Indeed, by Theorem \ref{thm:dfko}, we have $\ms{F}_{n,d} = \ms{F}_{n,d}\cap \ms{J}_{n,k(d,\eta),\eta}$ for any $\eta>0$, where $k(d,\eta) = \big\lceil2^{Cd}/\eta^2\big\rceil$. Plugging this in the bound \eqref{eq:junta} and optimizing over $\eta$, we deduce that there exists a universal constant $C >0$ such that
\begin{equation} \label{eq:new-EI}
\forall \ \e,\delta\in(0,1), \qquad \msf{Q}_r(\ms{F}_{n,d},\e,\delta) \leq \frac{2^{Cd^2}}{\e^{2d+1}} \log\left(\frac{n}{\delta}\right).
\end{equation}
It is worth emphasizing that, while \eqref{eq:new-EI} captures the correct dependence on the dimension $n$, it is asymptotically worse than the bounds \eqref{eq:EI}, \eqref{eq:EI2} both as $d\to\infty$ and as $\e\to0^+$. 

On the other hand, inequality \eqref{eq:EI2} is a special case of the bound \eqref{eq:upper}, up to lower order terms depending only on the degree $d$. Let $B_d=B_d^{\{\pm1\}}$ the Bohnenblust--Hille constant of the discrete hypercube. If for a function $f\in\ms{F}_{n,d}$ we define $\ms{S}(f)$ to be the collection of subsets $S$ of $\{1,\ldots,n\}$ for which $|\hat{f}(S)|\geq\e^{\frac{d+1}{2}}B_d^{-d}$, then we have
\begin{equation}
\#\ms{S}(f) \leq \e^{-d} B_d^{\frac{2d^2}{d+1}}\sum_{S\in\ms{S}(f)} |\hat{f}(S)|^{\frac{2d}{d+1}} \leq \frac{B_d^{2d}}{\e^d}
\end{equation}
and
\begin{equation}
\sum_{S\notin\ms{S}(f)} \hat{f}(S)^2 \leq \e B_d^{-\frac{2d}{d+1}}  \sum_{S\notin\ms{S}(f)} |\hat{f}(S)|^{\frac{2d}{d+1}} \leq \e
\end{equation}
by two applications of the Bohnenblust--Hille inequality. Thus \eqref{eq:EI2} follows from \eqref{eq:upper} with $t=0$.
\end{remark}

To prove Corollary \ref{cor:robust-boolean}, we will use a deep junta theorem of Bourgain \cite[Proposition]{Bou02}. The quantitative version which we employ below follows from \cite[Theorem~7.1]{KN06} (see also \cite{KO12,DJSTW15}).

\begin{theorem} \label{thm:bourgain}
Fix $n,d\in\N$ and $t\in(0,1)$.  For any Boolean function $f\in\ms{B}_{n,d}(t)$ and any parameter $\eta\geq \exp\big(C\sqrt{\log(2/t)\log\log d}\big) \big(t\sqrt{d}+\tfrac{1}{2^d}\big)$ there exists a collection of subsets $\ms{S}(f)$ of $\{1,\ldots,n\}$ with $\#\ms{S}(f) \leq 2^{O(d^2)}$ such that the spectrum of $f$ is $\eta$-concentrated on $\ms{S}(f)$.
\end{theorem}

To see how Theorem \ref{thm:bourgain} follows from \cite[Theorem~7.1]{KN06}, choose $\beta=2^{-\Omega(d)}$ in that statement and consider $\ms{S}(f)$ to be the collection of subsets $S\subseteq\{1,\ldots,n\}$ with $|S|\leq d$ and $S\subseteq J_\beta$. The fact that $\#\ms{S}(f)\leq 2^{O(d^2)}$ then follows since $|J_\beta|\leq 2^{O(d)}$ and the Fourier concentration property on $\ms{S}(f)$ follows from the conclusion of \cite[Theorem~7.1]{KN06}.  We note that the lower order terms on the size of $\eta$ with respect to $t,d$ can be removed from Theorem \ref{thm:bourgain} in view of a result of Kindler and O'Donnell \cite{KO12} at the expense of a worse dependence of $\#\ms{S}(f)$ on $d$.

\begin{proof} [Proof of Corollary \ref{cor:robust-boolean}]
Let $f\in\ms{B}_{n,d}(t)$. If $\eta \geq t^{1+o(1)}d^{\frac{1}{2}+o(1)}$ in the precise sense of Theorem \ref{thm:bourgain}, we have that the spectrum of $f$ is $\eta$-concentrated on a collection of subsets with cardinality $2^{O(d^2)}$. Therefore the conclusion follows from Theorem \ref{thm:upper} since also $t=o(\eta)$ as $t\to0^+$.
\end{proof}

\begin{remark}
It is worth emphasizing that the constraint $\eta\geq t\sqrt{d}$ which follows from \cite{Bou02,KO12} is in some sense optimal if one wishes to learn the class $\ms{B}_{n,d}(t)$ from logarithmically many samples. A linear threshold function (LTF) is a Boolean function of the form $f(x)=\mathrm{sign}\langle x,\theta\rangle$, where $x\in\{-1,1\}^n$ and $\theta\in\mb{S}^{n-1}$ is a fixed vector. A well-known theorem of Peres \cite{Per21} (see also \cite{BKS99}) asserts that any LTF on $n$ variables belongs in $\ms{B}_{n,\Omega(1/t^2)}(t)$ for every $t\in(0,1)$.  We shall argue that there exist $2^{\Omega(n)}$ LTFs which are pairwise $\Omega(1)$-apart which, in view of Proposition \ref{prop:QM},  will imply that the class of LTFs requires at least $\Omega(n)$ samples to be learned with accuracy $\tfrac{1}{4}$ and confidence $\tfrac{3}{4}$. Equivalently, we will show that there exist $N=2^{\Omega(n)}$ vectors $\theta_1,\ldots,\theta_N\in\mb{S}^{n-1}$ such that
\begin{equation} \label{cub}
\forall \ i\neq j, \qquad \big\|\mathrm{sign}\langle x,\theta_i\rangle - \mathrm{sign}\langle x,\theta_j\rangle \big\|_{L_2}^2 = 8\mb{P}\big\{ \big(\langle x,\theta_i\rangle, \langle x,\theta_j\rangle\big) \in U\big\} = \Omega(1),
\end{equation}
where $U$ is the second quadrant $\{(s,t): \ s\leq0\leq t\}$. The corresponding estimate in Gauss space follows from classical computations (see, e.g., \cite[Lemme~1]{Kri79}) as
\begin{equation} \label{gauss}
\forall \ u,v\in\mb{S}^{n-1}, \qquad \big\|\mathrm{sign}\langle g,u\rangle - \mathrm{sign}\langle g,v\rangle \big\|_{L_2}^2=2-2\mb{E}\big[ \mathrm{sign}\big(\langle g,u\rangle \cdot \langle g,v\rangle\big) \big] = 2- \frac{2}{\pi}\arcsin\langle u,v\rangle,
\end{equation}
where $g\sim N(0,\msf{Id}_n)$ is a standard Gaussian random vector, and thus it suffices to choose the vectors $\{\theta_i\}_{i=1}^N$ to form an $\Omega(1)$-net in the unit sphere. To pass from the Gaussian statement implied by \eqref{gauss} to the corresponding discrete inequality \eqref{cub} we shall use a classical (multivariate) Berry--Esseen theorem (see, e.g.,  \cite[Theorem~1.1]{Ben04}). In order to apply this result to the random vectors $(\langle x,\theta_i\rangle, \langle x,\theta_j\rangle)$,  it suffices to find an $\Omega(1)$-separated set $\{\theta_i\}_{i=1}^N$ in $\mb{S}^{n-1}$ with $N=2^{\Omega(n)}$ points such that $\|\theta_i\|_{\ell_\infty^n} \leq \tau$ for some small enough universal constant $\tau>0$. The existence of such a set can be proven by the probabilistic method in view of standard concentration estimates of $\ell_p^n$-norms on $\ell_q^n$-spheres, see for instance \cite[Remark~2 in p.~223]{SZ90} and \cite[Theorem~1]{AdRBV98}.
\end{remark}

Finally, we prove Corollary \ref{cor:circuits} on the complexity of constant depth circuits.

\begin{proof} [Proof of Corollary \ref{cor:circuits}]
By the main result of \cite{Has01} (which slightly improves \cite[Main Lemma]{LMN93}; see also the exposition in \cite[Section~4.5]{O'Do14}), every $f\in\ms{C}_{n,d,s}$ also belongs in $\ms{B}_{n,m(d,s,t)}(t)$ for every $t>0$, where $m(d,s,t) = O(\log(s/t))^{d-2}\cdot\log s\cdot \log(1/t)$. The conclusion follows by choosing $\eta=\e$ and $t$ small enough such that $\e\geq t^{1+o(1)} O(\log(s/t))^{d/2+o(1)}$ and applying Corollary \ref{cor:robust-boolean}.
\end{proof}

\subsection*{Running time considerations} While all the results of this section estimate the query complexity of various concept classes on the hypercube, it is worth emphasizing that they are also algorithmic in nature. For instance, the algorithm of Theorem \ref{thm:upper} (which generalizes the one of \cite{EI21}) has effectively the same running time (at least for constant $\e>0$) as the classical algorithm of Linial, Mansour and Nisan \cite{LMN93}, yet it offers an exponential improvement to the number of queries which are required as input.


\section{Exact learning}  \label{sec:4}

In this section, we shall prove Theorem \ref{thm:exact}. We start with the deterministic case.
\begin{proof} [Proof of \eqref{eq:exact}]
Let $Q=\msf{Q}(\ms{F}_{n,d},0)$. For the upper bound on $Q$, consider an enumeration $X_1,\ldots,X_k$ of the points in the closed Hamming $\mathrm{Ball}({\bf 1}, d)$, where ${\bf 1}=(1,\ldots,1)$ and $k=\sum_{j=0}^d \binom{n}{j}$.
\smallskip

\noindent {\it Claim.} If $f\in\ms{F}_{n,d}$,  then the values $f(X_1),\ldots,f(X_k)$ completely determine $f$.

\smallskip

\noindent {\it Proof.} As usual, for $i\in\{1,\ldots,n\}$ we denote by
\begin{equation} \label{eq:defpart}
\forall \ x\in\{-1,1\}^n, \qquad \partial_i f(x) \eqdef \frac{f(x_1,\ldots,x_i,\ldots,x_n) - f(x_1,\ldots,-x_i,\ldots,x_n)}{2}
\end{equation}
the discrete partial derivative of $f$. It is straightforward to see that if $f = \sum_S c_s w_S$, then
\begin{equation}
\forall \ x\in\{-1,1\}^n,\qquad \partial_i f(x) = \sum_{S: \ i\in S} c_S w_S(x).
\end{equation}
In particular, if $f$ has degree at most $d$ and $S=\{i_1,\ldots,i_d\}$ has cardinality $d$, then
\begin{equation}
\partial_{i_1}\circ\cdots\circ\partial_{i_d} f({\bf 1}) = c_S.
\end{equation}
On the other hand, \eqref{eq:defpart} implies that $\partial_{i_1}\circ\cdots\circ\partial_{i_d} f({\bf 1})$ is a linear combination of $f(X_1),\ldots,f(X_k)$. In other words, knowing $f(X_1),\ldots,f(X_k)$, we can reconstruct the top-order Walsh coefficients $\{c_S\}_{|S|=d}$. To conclude, we consider the function $f-\sum_{|S|=d} c_S w_S$ and iterate. \hfill$\Box$

\smallskip

Therefore, the claim implies that if the algorithm queries the values of $f$ at $X_1,\ldots,X_k$, then the function can be fully reconstructed, thus proving that $Q\leq k=\sum_{j=0}^d \binom{n}{j}$.

The lower bound is a simple dimension counting argument. Assume, for contradiction, that $\ms{F}_{n,d}$ can be learned exactly using $Q$ queries where $Q<k=\sum_{j=0}^d \binom{n}{j}$. Then, for any fixed points $X_1,\ldots,X_Q\in\{-1,1\}^n$, the linear system
\begin{equation} \label{eq:system}
\forall \ r=1,\ldots,Q,\qquad \sum_{|S|\leq d} c_S w_S(X_r) = 0
\end{equation}
with $k$ unknowns $\{c_S\}_{|S|\leq d}$ and $Q$ equations has at least one nonzero solution. In other words, there exists a nonzero function $g\in\ms{F}_{n,d}$ which vanishes on $\{X_1,\ldots,X_Q\}$.  If $X_1,\ldots,X_Q$ are the points queried by the algorithm in order to learn $g$, then \eqref{eq:system} shows the same points need to be queried to learn the zero function ${\bf 0}$. This is a contradictions as $H$ would produce the same hypothesis function for both and $g$ is not identically zero.
\end{proof}

Finally,  we prove the upper bound \eqref{eq:e=0} for the random example model.

\begin{proof} [Proof of \eqref{eq:e=0}]
For points $X_1,\ldots,X_Q$ on the hypercube, consider the (linear) evaluation operator $\Phi_{X_1,\ldots,X_Q}:\ms{F}_{n,d}\to\R^Q$ given by $\Phi_{X_1,\ldots,X_Q}(f) = (f(X_1),\ldots,f(X_Q))$. In order to prove the upper bound on the query complexity of the random example model without error, it suffices to show that if $Q$ is large enough and $X_1,\ldots,X_Q$ are independent and uniformly distributed random vectors on $\{-1,1\}^n$, then the operator $\Phi_{X_1,\ldots,X_Q}$ is injective with high probability. Indeed, if this is the case then the values of any function $f\in\ms{F}_{n,d}$ on a random sequence of samples uniquely determine $f$ with high probability and thus the function can be fully reconstructed by solving a system of linear equations with respect to its Walsh coefficients.

To show that $\Phi_{X_1,\ldots,X_Q}$ is injective with high probability, fix points $P_1,\ldots,P_q\in\{-1,1\}^n$ and let $X$ be a uniform random vector on the hypercube. Suppose that $\Phi_{P_1,\ldots,P_q}$ is not injective and choose a nonzero function $g\in\mathrm{ker}\Phi_{P_1,\ldots,P_q}$.  Then, we have
\begin{equation} \label{eq:pr-dec}
\mb{P}\big\{\mathrm{dim\ ker}\Phi_{P_1,\ldots,P_q,X} < \mathrm{dim\ ker}\Phi_{P_1,\ldots,P_q} \big\} \geq \mb{P}\{g(X)\neq0\}\geq\frac{1}{2^d},
\end{equation}
where the last inequality is a classical property satisfied by nonzero functions of degree at most $d$ which can be proven inductively, see \cite[Lemma 3.5]{O'Do14}.

To conclude the proof, consider $Q>k\eqdef \sum_{j=0}^d \binom{n}{j}$ and let $X_1,\ldots,X_Q$ be independent uniformly random points on the hypercube.  Suppose that the operator $\Phi_{X_1,\ldots,X_Q}$ is not injective. Then, at least $Q-k+1$ steps in the following chain of inequalities  are in fact equalities:
\begin{equation}
k \geq \mathrm{dim\ ker}\Phi_{X_1} \geq \mathrm{dim\ ker}\Phi_{X_1,X_2} \geq \cdots \geq \mathrm{dim\ ker}\Phi_{X_1,\ldots,X_Q}.
\end{equation}
By inequality \eqref{eq:pr-dec} and the independence of $X_1,\ldots,X_Q$, we deduce that
\begin{equation*}
\mb{P}\big\{\Phi_{X_1,\ldots,X_Q} \ \mbox{is not injective} \big\} \leq \binom{Q}{Q-k+1} \big(1-\tfrac{1}{2^d}\big)^{Q-k+1} \leq Q^{k-1}\big(1-\tfrac{1}{2^d}\big)^{Q-k+1} \leq (2Q)^{k-1}\big(1-\tfrac{1}{2^d}\big)^{Q}.
\end{equation*}
Choosing $Q = C2^d k \log\big(\frac{k}{\delta}\big)$ for a large enough universal constant $C>1$ ensures that $\Phi_{X_1,\ldots,X_Q}$ is injective with probability at least $1-\delta$,  thus completing the proof as $k\leq(d+1)n^d$.
\end{proof}

\begin{remark}
We point out that the query complexity estimate \eqref{eq:e=0} can be realized algorithmically. At every step of the algorithm, one has to compute the rank of the matrix $\Phi_{X_1,\ldots,X_q}$ until it becomes full-rank. Then, the unknown function $f$ can be recovered by solving a system of linear equations with respect to its Walsh coefficients.
\end{remark}




\section*{Acknowledgments} 

We are grateful to Roman Vershynin for providing many helpful pointers to the literature and Srinivasan Arunachalam for constructive feedback. We also thank the anonymous referees for their detailed comments which helped improved the presentation of the paper.


\bibliographystyle{amsplain}
\providecommand{\bysame}{\leavevmode\hbox to3em{\hrulefill}\thinspace}
\providecommand{\MR}{\relax\ifhmode\unskip\space\fi MR }
\providecommand{\MRhref}[2]{%
  \href{http://www.ams.org/mathscinet-getitem?mr=#1}{#2}
}
\providecommand{\href}[2]{#2}





\begin{dajauthors}
\begin{authorinfo}[alex]
  Alexandros Eskenazis\\
  CNRS, Institute de Math\'ematiques de Jussieu, Sorbonne Universit\'e\\
  Trinity College, University of Cambridge\\
    Paris, France and Cambridge, UK\\
  alexandros.eskenazis\imageat{}imj-prg\imagedot{}fr \\
  \url{https://www.alexandroseskenazis.com/}
\end{authorinfo}
\begin{authorinfo}[paata]
  Paata Ivanisvili\\
  Department of Mathematics, University of California, Irvine\\
  Irvine, CA, USA\\
 pivanisv\imageat{}uci\imagedot{}edu \\
  \url{https://sites.google.com/view/paata}
\end{authorinfo}
\begin{authorinfo}[lauritz]
  Lauritz Streck\\
  DPMMS, University of Cambridge\\
  Cambridge, UK\\
  ls909\imageat{}cam\imagedot{}ac\imagedot{}uk\\
  \url{https://lauritzstreck.com/}
\end{authorinfo}
\end{dajauthors}

\end{document}